\documentclass{article}

\usepackage[preprint]{neurips_2026}

\usepackage[utf8]{inputenc} %
\usepackage[T1]{fontenc}    %
\usepackage{hyperref}       %
\usepackage{url}            %
\usepackage{booktabs}       %
\usepackage{amsfonts}       %
\usepackage{nicefrac}       %
\usepackage{microtype}      %
\usepackage{xcolor}         %

\usepackage{graphicx}
\usepackage{subcaption}
\usepackage{amsmath}
\usepackage{amssymb}
\usepackage{mathtools}
\usepackage{amsthm}
\usepackage{algorithm}
\usepackage{algpseudocode}
\usepackage{float}

\usepackage[capitalize,noabbrev]{cleveref}

\theoremstyle{plain}
\newtheorem{theorem}{Theorem}[section]

\newtheorem{corollary}[theorem]{Corollary}
\theoremstyle{definition}

\theoremstyle{remark}

\newcommand{\AUC}{\operatorname{AUC}}

\newcommand{\Prob}{\mathbb{P}}

\title{Extending ROC Analysis to Uncertainty-Aware Risk Prediction with an Interval-Based AUC (iAUC)}

\author{%
  Yuqi Li \\
  Department of Biostatistics \& Bioinformatics\\
  Duke University\\
  \texttt{yuqi.li@duke.edu} \\
  \And
  Matthew M. Engelhard \\
  Department of Biostatistics \& Bioinformatics\\
  Duke University\\
  \texttt{m.engelhard@duke.edu}
}

\begin{document}

\maketitle

\begin{abstract}
In high-stakes risk prediction, interval-valued predictions provide a natural way to represent predictive uncertainty. However, standard evaluation tools such as the receiver operating characteristic (ROC) curve and the area under the curve (AUC) are designed for point-valued predictions and do not capture how uncertainty affects discrimination performance. To address this gap, we develop an uncertainty-aware ROC framework and corresponding interval-based AUC (iAUC) metrics for evaluating the performance of interval-valued risk predictions. The framework constructs two ROC-style curves with associated areas \(\AUC_L\) and \(\AUC_U\), and shows that these quantities induce a three-region decomposition of positive-negative rankings into confidently correct, ambiguous, and confidently incorrect orderings. Under valid class-conditional coverage, \(\AUC_L\) and \(\AUC_U\), together with the pairwise miscoverage rate, provide theoretical bounds on the optimal AUC, linking interval coverage to achievable discrimination. The framework is compatible with a range of methods for quantifying epistemic uncertainty in risk prediction, including Bayesian, bootstrap, and ensemble-based procedures, and provides a practical way to compare predictive models and interval-generation methods. Experiments on clinical benchmark data and MIMIC-IV mortality prediction show that models with similar point AUC can have different uncertainty-aware ranking profiles, and that the proposed decomposition can support risk-aware model selection and threshold-based abstention.

\end{abstract}

\section{Introduction}

In clinical medicine and other high-stakes risk prediction settings, predicting a single event probability can be insufficient, 
because it does not convey the model's confidence (or lack thereof) that its prediction accurately reflects true event risk. Consider stroke risk prediction, where decisions to initiate treatment such as statin therapy are guided by estimated 10-year risk thresholds \citep{uspstf2022statin}. A patient whose risk is below a low-risk threshold (\textit{e.g.}, 10\% probability) may choose to defer treatment, while a patient whose risk is above a high-risk threshold (\textit{e.g.}, 20\% probability) would be strongly advised to begin medical therapy. For some patients, however, the model may be \textit{uncertain} whether its prediction should fall above or below a given threshold.

Indeed, increasingly it is recognized that accurately quantifying the model's \textit{epistemic uncertainty}, which commonly arises from limited training data \citep{mucsanyi2024benchmarking, lindenmeyer2025towards}, is essential to maintain trust in high-stakes settings \citep{lindenmeyer2025towards, kompa2021second}.
When the model was trained on few patients similar to the patient in question (\textit{i.e.}, the patient lies in a low-density or out-of-distribution region of covariate space), epistemic uncertainty will be high, leading to a wider risk interval. In contrast, when the patient lies in a high-density region of covariate space, an uncertainty-aware model will produce a narrower risk interval, allowing us to confidently determine whether risk falls above or below a given threshold. Figure~\ref{fig:intro_intervals} illustrates this distinction: interval predictions can identify cases where the estimated risk lies clearly above or below a decision threshold, as well as cases where uncertainty makes the threshold-based decision ambiguous.

As the preceding example implies, this work is concerned with evaluating uncertainty-aware models in \textit{binary risk prediction} settings, defined as follows: for a patient with covariates \(x\), there is an underlying event probability \(\eta(x)=\mathbb{P}(Y=1\mid X=x)\) that gives rise, probabilistically, to a binary event realization $y$. Moreover, a model's certainty about $\eta(x)$ may depend on how much relevant data are available to support the estimate. Interval-valued predictions provide a natural way to represent this uncertainty by reporting a range of plausible risk values rather than a single number.

\begin{figure}[tb]
\centering
\includegraphics[width=1\columnwidth]{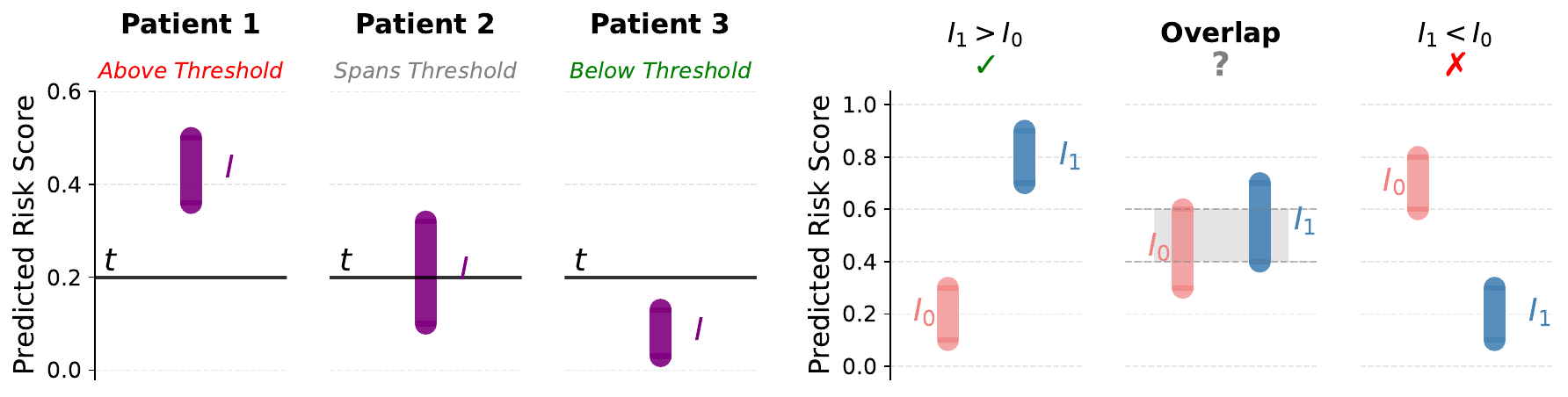}
\caption{Interval-based comparisons. \textbf{Left}: patient-level interval predictions relative to a decision threshold, where intervals entirely above or below the threshold support definitive threshold-based decisions and intervals spanning the threshold indicate ambiguity. \textbf{Right}: positive-negative interval comparisons used for AUC-style ranking, yielding correct, incorrect, or ambiguous orderings.}
\label{fig:intro_intervals}
\end{figure}

Motivated by the need for accurate uncertainty quantification (UQ), modern UQ methods can produce \emph{interval-valued predictions} using Bayesian credible intervals, bootstrap procedures, or ensemble methods \citep{efron1994bootstrap,gelman2013bayesian,gal2016dropout}. For each input \(X\), the model outputs an interval \(I(X)=[L(X),U(X)]\) intended to represent uncertainty about the underlying event risk. These intervals provide richer information than point estimates, but they also raise a basic evaluation question: how should discrimination be assessed when risk scores are intervals rather than points? Standard ROC analysis and AUC were designed for scalar scores. In practice, interval predictions are often collapsed to point estimates before AUC is computed \citep{lakshminarayanan2017simple,kuleshov2018accurate}, while calibration or coverage is evaluated separately from discrimination \citep{gneiting2007strictly,bracher2021evaluating}. This is unsatisfactory in binary risk prediction because the true event probability \(\eta(x)\) is unobserved, so interval coverage cannot generally be checked directly outside simulation. Moreover, point-AUC summaries discard interval width and overlap, and therefore cannot tell whether apparent discrimination comes from confident separation, uncertainty, or confident errors. This leaves no unified ROC-style framework for evaluating discrimination with interval-valued risk predictions.

In this paper, we develop an interval-based AUC framework for ROC analysis with interval-valued risk predictions. For two intervals \(I_1=[L_1,U_1]\) and \(I_0=[L_0,U_0]\) from the positive and negative classes, we extend the standard AUC convention that a positive-class instance is correctly ranked when it receives a higher risk score than a negative-class instance. For interval-valued predictions, there are three possible comparison outcomes (Figure~\ref{fig:intro_intervals}): (i) \textbf{confidently correct} rankings when \(L_1>U_0\), (ii) \textbf{confidently incorrect} rankings when \(U_1<L_0\), and (iii) \textbf{ambiguous} rankings when the intervals overlap. We show that these three outcomes induce a decomposition of interval-based ROC analysis. Specifically, we construct lower and upper ROC-style curves with areas \(\AUC_L\) and \(\AUC_U\), and prove the probabilistic interpretations \(\AUC_L=\mathbb{P}(I_1>I_0)\) and \(\AUC_U=1-\mathbb{P}(I_1<I_0)\).

The framework also supports risk-aware model comparison and threshold-based interpretation. It can compare models and interval-generation procedures by separating confidently correct rankings, confidently incorrect rankings, and ambiguous rankings. For threshold-based decisions, the same interval endpoints indicate when a prediction lies clearly above or below a decision threshold and when the model should abstain. We further derive bounds on the optimal AUC under class-conditional coverage assumptions, linking interval uncertainty to the theoretical limit of ranking performance. We validate the framework through benchmark and real-world clinical prediction experiments. On the Pima Indians Diabetes benchmark, we compare multiple predictive models and interval-generation procedures to show that models with similar point AUC can have different interval-based decompositions. On MIMIC-IV in-hospital mortality prediction, we show how the framework can support model comparison in a real clinical risk prediction task.

This work develops an interval-based ROC framework for interval-valued risk predictions, constructs lower and upper ROC-style curves with associated areas \(\AUC_L\) and \(\AUC_U\), and shows that they induce a three-region decomposition into confidently correct, confidently incorrect, and ambiguous rankings. The framework supports risk-aware model comparison and threshold-based abstention, yields theoretical bounds linking interval uncertainty to the optimal AUC, and empirically shows that models with similar point AUC can have different uncertainty-aware ranking profiles.

\section{Related work}

\paragraph{Classical ROC analysis and AUC.}
Receiver Operating Characteristic (ROC) analysis evaluates discrimination performance in binary classification. The ROC curve originated in signal detection theory and is now widely used for evaluating binary classifiers \citep{hanley1982meaning,fawcett2006introduction}. It plots the true positive rate against the false positive rate as a decision threshold varies, and the area under the curve (AUC) summarizes performance in a single scalar. A classical interpretation of AUC is the probability that a randomly drawn positive instance receives a higher score than a randomly drawn negative instance, with ties often assigned half credit \citep{hanley1982meaning,hand2001simple}. This ranking interpretation underlies its use in medical risk prediction \citep{pepe2003statistical}. Prior work has studied extensions such as partial AUC, cost-sensitive ROC analysis, and multi-class ROC analysis \citep{fawcett2006introduction}, as well as scoring rules that optimize AUC or characterize Bayes-optimal ranking performance \citep{cortes2004auc,clemencon2008ranking}. Fuzzy ROC introduces a gray zone on the scalar-score scale, within which observations may remain unclassified, and examines the resulting sensitivity--specificity trade-offs \citep{parmigiani2019fuzzy}. Our framework instead takes interval-valued risk predictions as inputs and defines ambiguity at the level of positive-negative interval comparisons, while preserving the classical ranking interpretation of AUC.

\paragraph{Methods producing interval-valued predictions.}
A variety of approaches produce interval-valued predictions for binary risk prediction tasks.
Bootstrap methods generate empirical prediction distributions by resampling the training data and constructing percentile-based confidence intervals \cite{efron1994bootstrap,diciccio1996bootstrap}.
Bayesian methods provide posterior predictive distributions from which credible intervals can be derived, including Bayesian generalized linear models, Gaussian processes, and Bayesian neural networks \cite{mackay1992practical, gelman2013bayesian,gal2016dropout}.
Ensemble methods estimate uncertainty through variability across models or repeated forward passes, yielding intervals via empirical quantiles \cite{lakshminarayanan2017simple}. Our framework is agnostic to the interval-generation mechanism and applies to any method that outputs interval-valued predictions. 

\paragraph{Evaluation of prediction intervals.}

Standard evaluation of prediction intervals focuses on calibration and sharpness, often following the principle of maximizing sharpness subject to calibration \citep{gneiting2007strictly}. These properties can be summarized through proper scoring rules, interval scores, or under- and over-coverage diagnostics \citep{gneiting2007strictly,bracher2021evaluating}. More closely related to discrimination, \citet{gray2021logistic} study logistic regression with imprecise input data and use the resulting prediction envelopes to obtain bounded ROC curves and AUC values. Our work considers the more general setting in which interval-valued predictions may be produced by arbitrary prediction and uncertainty-quantification methods. Standard interval-evaluation metrics assess uncertainty quantification for individual predictions, but they do not directly describe how interval uncertainty affects ranking discrimination. In binary risk prediction, this distinction is important, because individual-level coverage of the true posterior risk is generally unobserved, while discrimination is commonly evaluated through observable positive-negative comparisons.

\paragraph{Selective prediction and abstention.}
Selective prediction studies models that abstain when confidence is low, creating a trade-off between retention rate and accuracy among retained instances \citep{el2010selective,geifman2019selectivenet}. Most approaches use scalar confidence scores to reject individual predictions and focus on label-level decision policies. Our setting is different: interval-valued risk scores induce both pair-level ranking ambiguity and threshold-level decision ambiguity. The three-region decomposition evaluates confident and ambiguous positive-negative rankings, while the interval-threshold rule provides an instance-level abstention interpretation when a fixed decision threshold is available. This connects abstention-style reasoning with ROC-style discrimination without replacing existing selective prediction methods.

\paragraph{Research gap and our contribution.}
Existing work provides mature tools for evaluating scalar classifiers and for assessing the calibration and sharpness of prediction intervals. What remains missing is a ROC-style framework for evaluating performance of uncertainty-aware models when predictions are intervals rather than point scores. Standard AUC collapses intervals to point scores, and interval-scoring methods do not show whether ranking performance comes from confident separation, uncertain overlap, or confidently incorrect ordering. Our framework addresses this gap by quantifying these three ranking outcomes and connecting interval-based ranking evaluation to threshold-based abstention and theoretical bounds on the optimal AUC.

\section{Setup and theory}
\label{sec:theory}

\subsection{Problem setup and interval comparisons}
\label{subsec:setup}

We consider a binary risk prediction problem with labels \(Y\in\{0,1\}\), where \(Y=1\) denotes the positive class and \(Y=0\) the negative class. The goal is to assign a continuous risk score for the positive class, which may later be used for ranking or threshold-based decisions. We suppose the model outputs an interval-valued score \(I(x)=[L(x),U(x)]\), with \(L(x)\le U(x)\) for each input $x$. Let \(I_1=[L_1,U_1]\) and \(I_0=[L_0,U_0]\) denote random prediction intervals for independent draws from the positive and negative class distributions, respectively, where the subscripts indicate the true class to which the subject belongs.

We use strict non-overlap comparisons to rank interval-valued scores. For an interval \(I=[L,U]\) and a threshold \(t\), we write \(I>t\) if \(L>t\) and \(I<t\) if \(U<t\), requiring the entire interval to lie on one side of the threshold. For two intervals \(I_1=[L_1,U_1]\) and \(I_0=[L_0,U_0]\), we define
\[
I_1>I_0 \iff L_1>U_0,
\qquad
I_1<I_0 \iff U_1<L_0.
\]
Following the standard AUC convention, \(I_1>I_0\) is a correct interval ordering because the positive-class interval lies entirely above the negative-class interval, while \(I_1<I_0\) is an incorrect interval ordering. If neither condition holds, the intervals overlap and the ordering is ambiguous. Figure~\ref{fig:intro_intervals} illustrates these cases. This setup makes no assumptions about the distribution within each interval and uses only the reported endpoints. Consequently, pairwise comparisons are not weighted by the degree of overlap, and all overlapping pairs are treated as ambiguous.

\subsection{Lower and upper interval ROC curves}
\label{subsec:roc_curves}

Using interval--threshold comparisons, we define four rate functions, where TPR and FPR denote the true positive rate and false positive rate, respectively, for a threshold \(t\in\mathbb{R}\):
\begin{align}
\mathrm{TPR}_L(t) &= \mathbb{P}(L_1>t), 
& \mathrm{TPR}_U(t) &= \mathbb{P}(U_1>t), \\
\mathrm{FPR}_L(t) &= \mathbb{P}(L_0>t), 
& \mathrm{FPR}_U(t) &= \mathbb{P}(U_0>t).
\end{align}
By construction, \(\mathrm{TPR}_L(t)\le \mathrm{TPR}_U(t)\) and \(\mathrm{FPR}_L(t)\le \mathrm{FPR}_U(t)\) for all \(t\).

We construct two interval ROC curves by pairing rate functions with different levels of strictness. The lower curve plots \(\mathrm{TPR}_L(t)\) against \(\mathrm{FPR}_U(t)\), while the upper curve plots \(\mathrm{TPR}_U(t)\) against \(\mathrm{FPR}_L(t)\). The associated areas under these curves are denoted by

\begin{equation}
\AUC_L=\int \mathrm{TPR}_L(t)\,d[\mathrm{FPR}_U(t)],
\qquad
\AUC_U=\int \mathrm{TPR}_U(t)\,d[\mathrm{FPR}_L(t)].
\label{eq:iAUC_defs}
\end{equation}

The lower curve gives a strict assessment of discrimination: a positive interval must lie entirely above the threshold to contribute to \(\mathrm{TPR}_L\), while a negative interval contributes to \(\mathrm{FPR}_U\) if any part exceeds the threshold. The upper curve gives a permissive assessment: a positive interval contributes to \(\mathrm{TPR}_U\) if any part exceeds the threshold, while a negative interval contributes to \(\mathrm{FPR}_L\) only if its entire interval exceeds the threshold.

\subsection{Interval-based AUC: probabilistic and geometric interpretations}
\label{subsec:iAUC_geometry}
A key property of the classical AUC is its probabilistic interpretation:
\begin{equation}
    \AUC(f)=\mathbb{P}\{f(X_1)>f(X_0)\}+\frac{1}{2}\mathbb{P}\{f(X_1)=f(X_0)\},
\label{eq:point_AUC}
\end{equation}

where \(X_1\) and \(X_0\) are independently drawn positive and negative instances, and tied scores receive half credit \citep{hanley1982meaning,hand2001simple}. We show that this property extends naturally to interval-valued predictions, yielding analogous interpretations for \(\AUC_L\) and \(\AUC_U\).

\begin{theorem}[Lower interval AUC]\label{thm:aucl}
The lower interval AUC equals the probability that a randomly chosen positive interval is strictly above a randomly chosen negative interval:
\begin{equation}
    \AUC_L=\mathbb{P}(L_1>U_0)=\mathbb{P}(I_1>I_0).
\end{equation}
\end{theorem}

\begin{proof}
Let \(G_1(l)=\mathbb{P}(L_1\le l)\) and \(H_0(u)=\mathbb{P}(U_0\le u)\). Since \(\mathrm{TPR}_L(t)=\mathbb{P}(L_1>t)=1-G_1(t)\) and \(\mathrm{FPR}_U(t)=\mathbb{P}(U_0>t)=1-H_0(t)\), we have \(d[\mathrm{FPR}_U(t)]=-dH_0(t)\). Therefore,
\[
\AUC_L=\int (1-G_1(t))\,d(1-H_0(t))=\mathbb{E}_{U_0}\{1-G_1(U_0)\}=\mathbb{P}(L_1>U_0). \qedhere
\]
\end{proof}

\begin{theorem}[Upper interval AUC]\label{thm:aucu}
The upper interval AUC satisfies
\begin{equation}
    \AUC_U=1-\mathbb{P}(U_1<L_0)=1-\mathbb{P}(I_1<I_0).
\end{equation}
\end{theorem}
The proof follows a similar argument to Theorem~\ref{thm:aucl} and is provided in Appendix~\ref{app:proof_aucu}.

\begin{corollary}[Three-region decomposition]\label{cor:decomposition}
The pairwise comparison between a randomly chosen positive interval and a randomly chosen negative interval decomposes into three mutually exclusive regions:
\begin{equation}
    \mathbb{P}(I_1>I_0)+\mathbb{P}(\mathrm{overlap})+\mathbb{P}(I_1<I_0)=1,
\end{equation}
where overlap is defined by \(L_1\le U_0\) and \(U_1\ge L_0\). Equivalently, $\mathbb{P}(\mathrm{overlap})=\AUC_U-\AUC_L$.
\end{corollary}

\paragraph{Geometric interpretation.}
The two interval ROC curves give a geometric view of the three-region decomposition (Figure~\ref{fig:roc_geometry_decomposition}). The figure is generated from an illustrative bootstrap interval example, and implementation details are provided in Appendix~\ref{app:empirical_validation}. The area below the lower curve equals \(\mathbb{P}(I_1>I_0)\), the area above the upper curve equals \(\mathbb{P}(I_1<I_0)\), and the remaining region has area \(\AUC_U-\AUC_L=\mathbb{P}(\mathrm{overlap})\). At a fixed threshold \(t\), the separation between the two curve points also has a local interpretation: the vertical difference is \(\mathrm{TPR}_U(t)-\mathrm{TPR}_L(t)=\mathbb{P}(L_1\le t\le U_1)\), and the horizontal difference is \(\mathrm{FPR}_U(t)-\mathrm{FPR}_L(t)=\mathbb{P}(L_0\le t\le U_0)\). Thus, the global gap between \(\AUC_L\) and \(\AUC_U\) reflects pairwise ambiguity, while the local separation at a threshold reflects class-conditional threshold ambiguity.

\begin{figure}[tb]
\centering
\includegraphics[width=0.9\linewidth,clip]{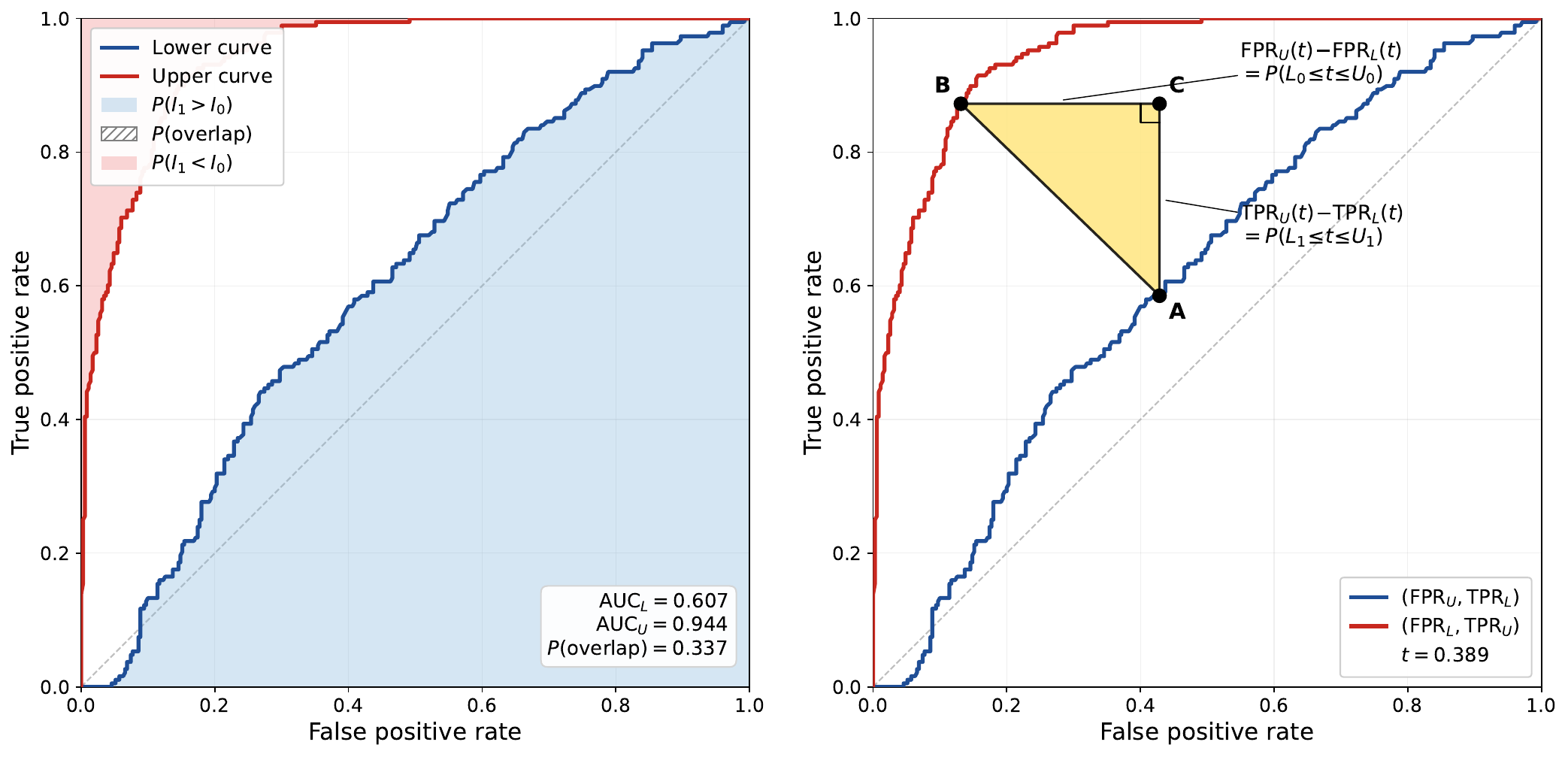}
\caption{Geometric interpretation of interval-based ROC analysis. (a) The lower curve plots \(\mathrm{TPR}_L\) versus \(\mathrm{FPR}_U\), and the upper curve plots \(\mathrm{TPR}_U\) versus \(\mathrm{FPR}_L\), separating the ROC plane into correct, incorrect, and ambiguous ordering regions. (b) At a fixed threshold \(t\), the two curve points form a right triangle whose legs measure positive-class (AC) and negative-class (BC) threshold ambiguity.}
\label{fig:roc_geometry_decomposition}
\end{figure}

\paragraph{Connection to classical AUC.}
When prediction intervals collapse to points, that is, \(L=U\) almost surely, the overlap probability vanishes and the partial ordering becomes total. In this case, \(\mathrm{TPR}_L(t)=\mathrm{TPR}_U(t)\) and \(\mathrm{FPR}_L(t)=\mathrm{FPR}_U(t)\), so \(\AUC_L=\AUC_U\) reduces to the classical AUC. Thus, the proposed framework recovers standard ROC analysis as a special case.

\subsection{Risk-aware model comparison and threshold abstention}
\label{subsec:selective}

The three-region decomposition can support risk-aware model comparison at the pairwise ranking level. For a fixed interval level, each model-generator pair can be summarized by \(P(I_1>I_0)\), \(P(\mathrm{overlap})\), and \(P(I_1<I_0)\). Since \(\AUC_L=P(I_1>I_0)\), one natural comparison rule is to prefer larger \(\AUC_L\) while controlling the probability of confidently incorrect rankings:
\begin{equation}
    \max_m P_m(I_1>I_0)
\qquad
\text{subject to}
\qquad
P_m(I_1<I_0)\le \epsilon,
\end{equation}

where \(m\) indexes the model-generator pair and \(\epsilon\) is an application-specific tolerance. This rule is close in spirit to standard AUC-based model comparison because it still compares positive-negative ranking performance. However, it separates confidently incorrect rankings from ambiguous comparisons rather than collapsing them into one scalar score. Models with excessive overlap are naturally penalized because overlap reduces \(P(I_1>I_0)=\AUC_L\).

For instance-level selective prediction, a fixed decision threshold \(t\) is needed. The same interval endpoints induce a threshold-abstention rule:
\[
L(x)>t \Rightarrow \widehat{Y}=1,\qquad U(x)<t \Rightarrow \widehat{Y}=0,\qquad L(x)\le t\le U(x) \Rightarrow \mathrm{abstain}.
\]

This rule answers a different question from pairwise model comparison. The pairwise rule evaluates how well a model ranks positive cases above negative cases, whereas threshold abstention decides whether to retain or abstain on an individual prediction at a chosen decision threshold. A subject is retained only when its interval lies entirely above or below \(t\), and is abstained when the interval contains \(t\). Among retained subjects, standard operating characteristics such as TPR or FPR can be computed conditional on retention, and should be reported together with the abstention rate.

\section{Coverage-based bounds on the optimal AUC}
\label{sec:optimal_auc}

In this section, we study how interval-based AUC quantities relate to the best ranking performance achievable under the data distribution. Under class-conditional coverage, strict interval orderings provide partial information about the optimal ranking. This yields coverage-based bounds on the optimal AUC in terms of \(\AUC_L\), \(\AUC_U\), and the pairwise miscoverage rate.

\subsection{Optimal AUC and Bayes-optimal ranking}
\label{subsec:optimal_auc_def}

Let \((X,Y)\sim P\) with labels \(Y\in\{0,1\}\), and let \(\eta(x)=\Prob(Y=1\mid X=x)\) denote the true posterior probability, also known as the regression function. Using the point-AUC definition in Eq.~\eqref{eq:point_AUC}, the \emph{optimal AUC} is the maximum ranking performance attainable under the true data distribution:
\begin{equation}
    \AUC^*:=\sup_f \AUC(f),
\end{equation}

where the supremum is over all measurable scoring functions \(f:\mathcal{X}\to\mathbb{R}\). A fundamental result in bipartite ranking shows that the Bayes posterior \(\eta(x)\) induces the optimal ranking. As shown by \citet{clemencon2008ranking}, any strictly monotone transformation of \(\eta(x)\) maximizes the ROC curve and achieves \(\AUC^*\). Therefore, under the usual no-tie convention for continuous risk scores,
\begin{equation}
    \AUC^*=\Prob\bigl(\eta(X_1)>\eta(X_0)\bigr).
\end{equation}
This value represents the Bayes-optimal limit of discrimination under the intrinsic overlap of the class-conditional distributions.

\subsection{Bounds under class-conditional coverage}
\label{subsec:optimal_auc_bound}

Suppose the interval-valued prediction \(I(x)=[L(x),U(x)]\) is designed to cover the posterior risk \(\eta(x)\). For an independent positive-negative pair, with independence conditional on class labels, define the class-conditional miscoverage rates, pairwise coverage event, and pairwise miscoverage probability as
\begin{equation}
\begin{aligned}
\alpha_i &= \mathbb{P}\{\eta(X)\notin I(X)\mid Y=i\},\qquad i\in\{0,1\},\\
E_{\mathrm{pair}} &= \{\eta(X_1)\in I(X_1)\}\cap\{\eta(X_0)\in I(X_0)\},\\
p_{\mathrm{pair}} &= \mathbb{P}(E_{\mathrm{pair}}^c)=1-(1-\alpha_1)(1-\alpha_0)=\alpha_1+\alpha_0-\alpha_1\alpha_0.
\end{aligned}
\label{eq:coverage_quantities}
\end{equation}

On the event \(E_{\mathrm{pair}}\), strict interval ordering implies posterior ordering. In particular, \(I_1>I_0\) implies \(\eta(X_1)>\eta(X_0)\), and \(\eta(X_1)>\eta(X_0)\) rules out the strict reverse ordering \(I_1<I_0\). This gives the following coverage-based bounds.

\begin{theorem}[Coverage-based bounds on the optimal AUC]
\label{thm:main_bound}
If \(I(x)\) has class-conditional miscoverage rates \(\alpha_1\) and \(\alpha_0\), then
\begin{equation}
    \AUC_L-p_{\mathrm{pair}}\le \AUC^*\le \AUC_U+p_{\mathrm{pair}},
\end{equation}
where \(p_{\mathrm{pair}}=\alpha_1+\alpha_0-\alpha_1\alpha_0\).
\end{theorem}

The proof is provided in Appendix~\ref{app:proof_main_bound}. This result gives a principled coverage-based range for \(\AUC^*\). When the interval gap \(\AUC_U-\AUC_L\) and the pairwise miscoverage \(p_{\mathrm{pair}}\) are small, the range localizes the Bayes-optimal discrimination. When interval overlap is substantial, the range quantifies the ranking uncertainty that remains under the available interval information.

\section{Experiments}
\label{sec:experiments}

We evaluate whether the proposed interval-based AUC framework reveals ranking information hidden by standard point-based AUC and whether this information can guide uncertainty-aware model selection. Specifically, we ask whether the three-region decomposition remains stable across different predictive models and interval-generation procedures, whether it distinguishes model-generator pairs with similar point AUC, and whether it supports a simple selection rule that controls confidently incorrect rankings while favoring a large confidently correct region. We address these questions on both a benchmark clinical dataset and a real-world MIMIC-IV mortality prediction task. For each model-generator pair, we report the standard point AUC and the three regions under the interval-based AUC framework: \(P(I_1>I_0)\), \(P(\mathrm{overlap})\), and \(P(I_1<I_0)\). These regions correspond to confidently correct, ambiguous, and confidently incorrect rankings, respectively. Numerical validation of the AUC probability identities is provided in Appendix~\ref{app:empirical_validation}.

\subsection{Benchmark model-generator comparison}
\label{subsec:benchmark}

\paragraph{Setup.}
We first study the Pima Indians Diabetes benchmark, a binary clinical prediction task with eight clinical features, 
across combinations of base models and interval-generation (\textit{i.e.}, uncertainty quantification) methods. The base models include logistic regression (LR), random forest (RF), histogram gradient boosting machine (HistGBM), multilayer perceptron (MLP), and explainable boosting machine (EBM). The interval-generation methods include bootstrap intervals, deep ensembles, Bayesian bootstrap intervals, and model-specific uncertainty baselines when available. For each fitted model-generator pair, we compute a point AUC using the average predicted risk across resamples or ensemble members, and we compute the three-region decomposition over nominal interval levels ranging from $0\%$ to $99\%$. Full experimental details are provided in Appendix~\ref{app:benchmark_details}.

\paragraph{Three-region profiles across interval levels.}
Figure~\ref{fig:pima_sweep} shows representative model-generator pairs. The point AUCs are close, ranging from 0.810 to 0.831, so the point-AUC alone provides little information about how uncertainty affects ranking. This is distinct from uncertainty in the AUC estimates themselves, which could be assessed separately by bootstrapping the test set. In contrast, the interval-based decomposition reveals different uncertainty-aware ranking profiles. For the same base model, changing the interval-generation procedure can substantially change the decomposition: HistGBM with bootstrap intervals converts much of its ranking mass into overlap at high interval levels, whereas HistGBM with deep-ensemble intervals retains a larger \(P(I_1>I_0)\) region. A similar contrast appears for EBM. Across models, the same point-AUC range also hides different trade-offs among confidently correct, ambiguous, and confidently incorrect rankings. These results show that interval-based evaluation should treat the base model and interval generator as a combined system.

\begin{figure}[t]
    \centering
    \includegraphics[width=0.9\linewidth]{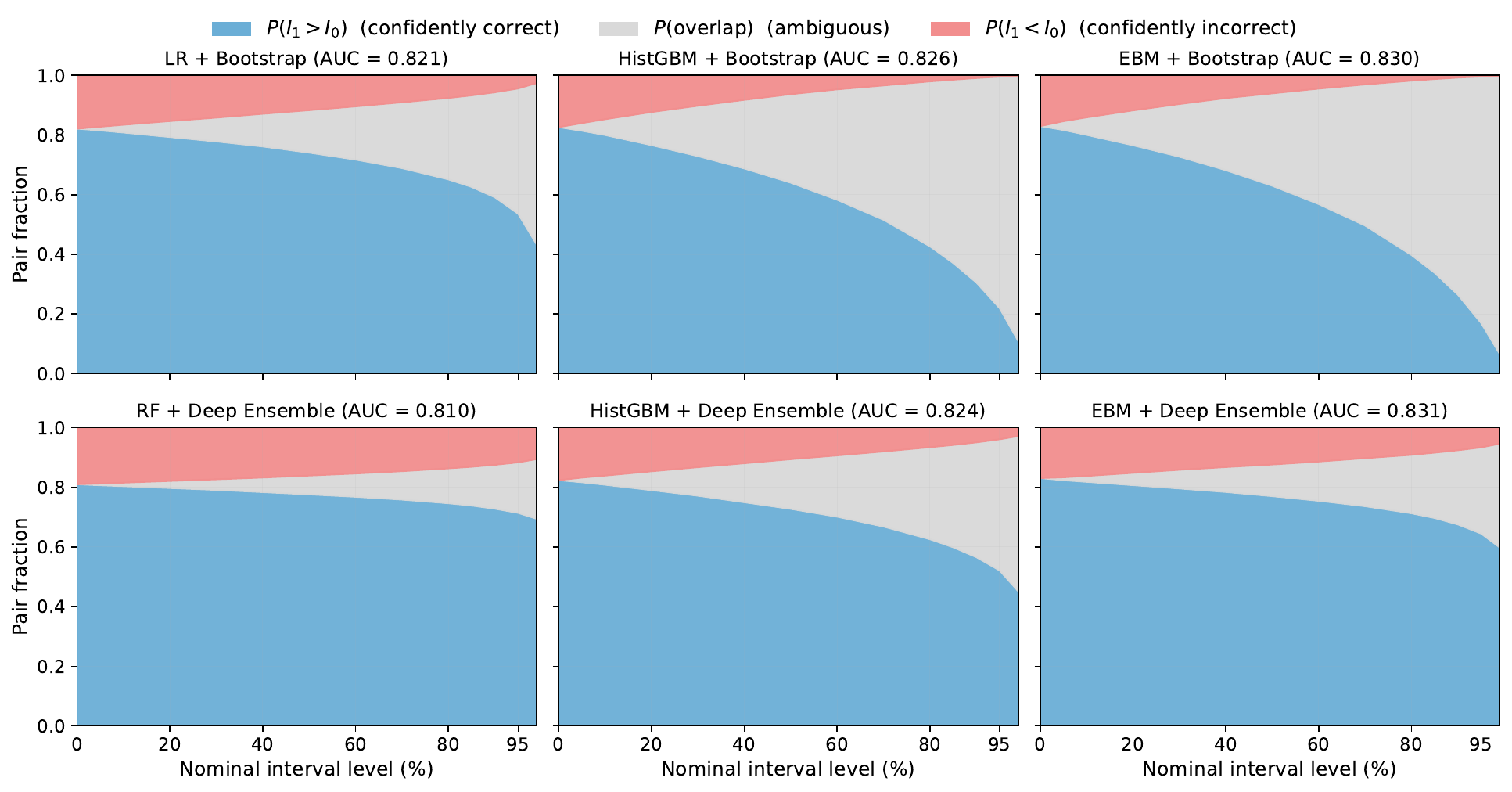}
    \caption{Three-region decomposition on the Pima Indians Diabetes benchmark across nominal interval levels. Each panel shows a representative model-generator pair.}
    \label{fig:pima_sweep}
\end{figure}

\paragraph{Fixed-level model comparison.}
Table~\ref{tab:pima_selected} reports the same quantities at the 90\% nominal interval level. This fixed-level view illustrates the model-comparison rule in Section~\ref{subsec:selective}. For example, EBM with bootstrap intervals and EBM with deep-ensemble intervals have nearly identical point AUCs (0.830 vs. 0.831), but their interval-based decompositions are very different: the bootstrap version has a larger overlap region (0.728) and a smaller confidently correct region (0.265), whereas the deep-ensemble version has a smaller overlap region (0.249) and a larger confidently correct region (0.675). Suppose an application requires \(P(I_1<I_0)\le 0.10\). Under this tolerance, all selected model-generator pairs in Table~\ref{tab:pima_selected} except RF with deep-ensemble intervals are eligible, and EBM with deep-ensemble intervals is selected because it has the largest \(\AUC_L=P(I_1>I_0)\). Thus, two model-generator pairs that appear nearly equivalent under point AUC can lead to different conclusions once confidently correct, confidently incorrect, and ambiguous rankings are separated. Full fixed-level results are provided in Appendix~\ref{app:pima_full_results}.

\begin{table}[t]
    \centering
    \caption{Selected Pima Indians Diabetes model-generator comparisons at 90\% nominal interval level.}
    \label{tab:pima_selected}
    \begin{tabular}{llrrrr}
\toprule
Model & Generator & Point AUC & $P(I_1>I_0)$ & $P(\mathrm{overlap})$ & $P(I_1<I_0)$ \\
\midrule
LR & Bootstrap & 0.821 & 0.591 & 0.352 & 0.057 \\
RF & Deep Ensemble & 0.810 & 0.728 & 0.147 & 0.125 \\
HistGBM & Bootstrap & 0.826 & 0.305 & 0.685 & 0.009 \\
HistGBM & Deep Ensemble & 0.824 & 0.566 & 0.385 & 0.049 \\
EBM & Bootstrap & 0.830 & 0.265 & 0.728 & 0.007 \\
EBM & Deep Ensemble & 0.831 & 0.675 & 0.249 & 0.076 \\
\bottomrule
\end{tabular}

\end{table}

\subsection{MIMIC-IV mortality prediction}
\label{subsec:mimic}

\paragraph{Setup.}
We next apply the framework to MIMIC-IV in-hospital mortality prediction \citep{johnson2023mimic} as a real-world clinical risk prediction example. We use first-24-hour demographic, laboratory, and clinical features and compare four predictive models: L2-regularized logistic regression, random forest, histogram gradient boosting, and multilayer perceptron. Prediction intervals are constructed using bootstrap resampling, and all interval-based metrics are evaluated on the held-out test set at the 90\% nominal interval level. Additional details and results are provided in Appendix~\ref{app:mimic_details}.

\begin{figure}[tb]
\centering
\begin{subfigure}{0.55\linewidth}
    \centering
    \includegraphics[width=\linewidth]{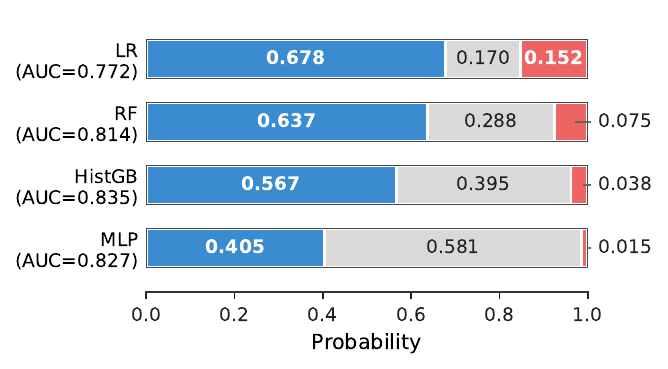}
    \caption{MIMIC-IV mortality prediction.}
    \label{fig:mimic_bar}
\end{subfigure}
\hfill
\begin{subfigure}{0.44\linewidth}
    \centering
    \includegraphics[width=\linewidth]{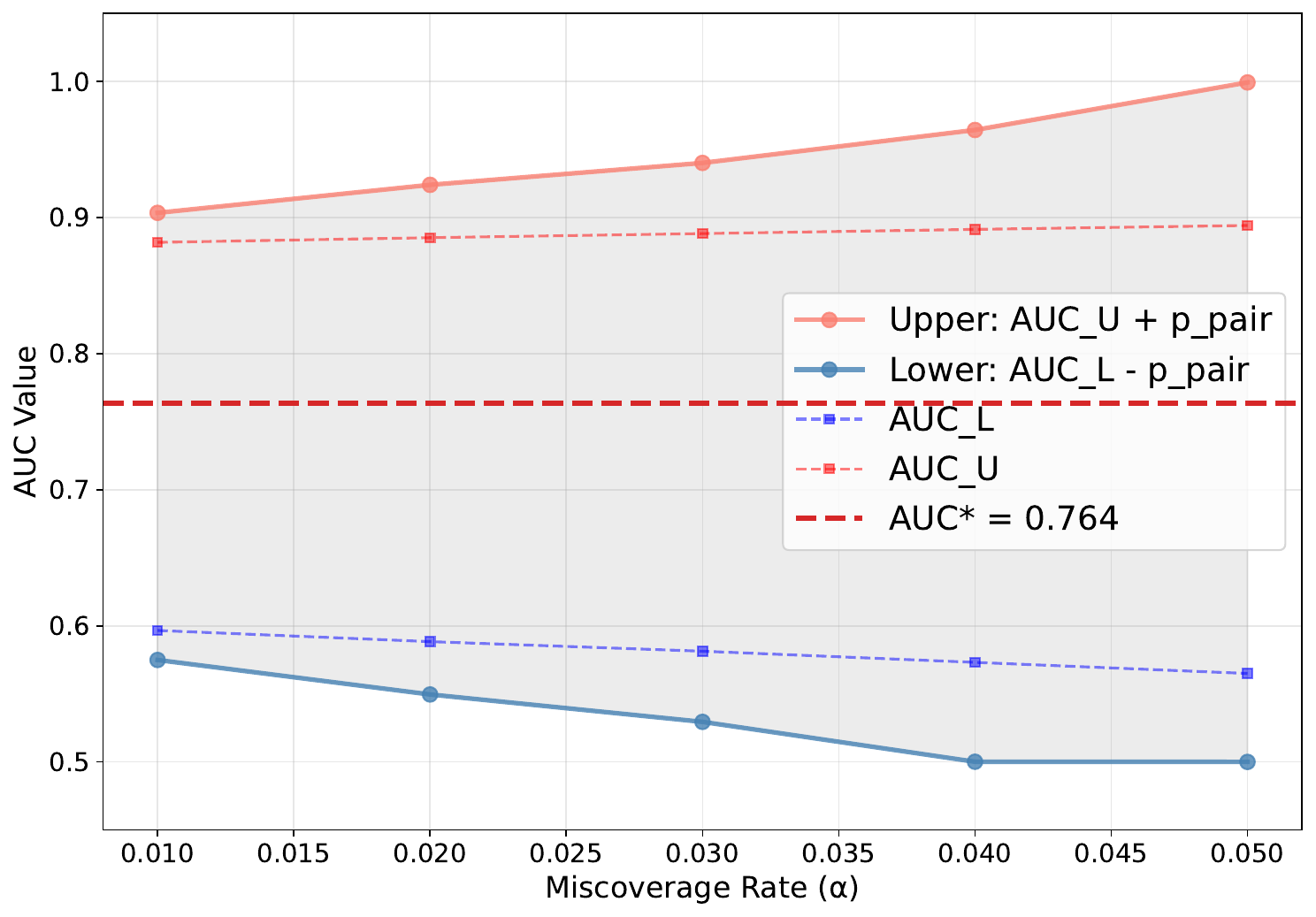}
    \caption{Synthetic bound validation.}
    \label{fig:bound}
\end{subfigure}
\caption{Additional empirical analyses. (a) MIMIC-IV mortality prediction at the 90\% nominal interval level. Each bar decomposes positive-negative pairs into confidently correct (blue), ambiguous (gray), and confidently incorrect (red) rankings. (b) Synthetic validation of the coverage-based bound on the optimal AUC. The shaded region denotes \([\AUC_L-p_{\mathrm{pair}}, \AUC_U+p_{\mathrm{pair}}]\).}
\label{fig:mimic_and_bound}
\end{figure}

\paragraph{Results.}
Figure~\ref{fig:mimic_and_bound}(a) summarizes the three-region decomposition across MIMIC-IV mortality prediction models. The point-AUC ranking differs from the interval-aware profile: HistGB has the largest point AUC, but RF has a larger confidently correct region and a smaller confidently incorrect region. The decomposition also reveals different uncertainty profiles behind similar point-score performance. LR retains the largest confidently correct region, but has the largest confidently incorrect fraction, \(P(I_1<I_0)=0.152\). MLP has the smallest confidently incorrect fraction, \(P(I_1<I_0)=0.015\), but much of its apparent discrimination is converted into overlap, with \(P(\mathrm{overlap})=0.581\). Thus, the framework distinguishes confident errors from ambiguity, which are not separated by a standard point-AUC summary. For model selection, suppose an application requires \(P(I_1<I_0)\le 0.10\). LR is excluded because its confidently incorrect fraction exceeds this tolerance. Among the remaining models, RF has the largest \(\AUC_L=P(I_1>I_0)=0.637\), so the interval-aware rule selects RF rather than HistGB, which has the largest point AUC. Full numeric results, model-selection summaries, and threshold-abstention analyses are provided in Appendix~\ref{app:mimic_details}.

\subsection{Synthetic validation of optimal AUC bounds}
\label{subsec:bound_exp}

Finally, we perform a synthetic validation of Theorem~\ref{thm:main_bound} using a one-dimensional Gaussian model in which the true posterior \(\eta(x)\) is known and the Bayes-optimal AUC can be computed directly. As shown in Figure~\ref{fig:mimic_and_bound}(b), the Bayes-optimal AUC is contained within the coverage-based bound across evaluated miscoverage rates. The interval widens as pairwise miscoverage increases, reflecting the conservative nature of the coverage correction. Full setup details are provided in Appendix~\ref{app:bound_validation}.

\section{Discussion and conclusion}

We introduce the iAUC framework, extending ROC analysis to interval-valued risk predictions. By constructing ROC-style curves with associated areas \(\AUC_L\) and \(\AUC_U\), the framework evaluates discrimination without collapsing uncertainty intervals to point scores. The central result is a three-region decomposition of positive-negative rankings into confidently correct, ambiguous, and confidently incorrect components. This decomposition preserves the ranking interpretation of classical AUC while making visible how interval uncertainty changes model discrimination. The framework also provides coverage-based bounds on the optimal AUC under class-conditional coverage assumptions.

The experiments show that point AUC can hide important differences in uncertainty-aware ranking behavior. On the Pima Indians Diabetes benchmark, model-generator pairs with similar point AUCs produced substantially different three-region profiles, indicating that the predictive model and interval-generation procedure should be evaluated together. On MIMIC-IV mortality prediction, the same decomposition showed that models differed not only in point AUC, but also in how much ranking mass was assigned to confident separation, ambiguity, and confidently incorrect ordering. These results illustrate how the framework can support risk-aware model comparison by controlling confidently incorrect rankings while favoring a large confidently correct region. The interval endpoints also provide a natural basis for threshold-based abstention.

This work also leaves several directions for future study. The current framework assumes ordered interval-valued risk scores in binary prediction. Extensions to multi-class, survival, or competing-risk settings will require additional definitions of ordering and ambiguity. The framework evaluates intervals produced by an uncertainty quantification procedure, but does not itself guarantee that the intervals have valid coverage. Finally, different interval-generation methods may encode different sources of uncertainty, and understanding how these choices affect iAUC profiles remains an important empirical question. Overall, iAUC provides a principled approach for evaluating uncertainty-aware risk prediction models by connecting ROC analysis, interval uncertainty, risk-aware model comparison, and threshold-based abstention.

\newpage

\bibliographystyle{plainnat}
\bibliography{ref.bib}

\newpage
\appendix
\section{Proofs}

\subsection{Proof of Theorem~\ref{thm:aucl} and~\ref{thm:aucu}}
\label{app:proof_aucu}

For simplicity, we state the result under the no-boundary-ties condition \(\mathbb{P}(U_1=L_0)=0\), which holds for continuous endpoint distributions. Boundary ties can be handled by the same half-credit convention used in the classical AUC.

We prove Theorem~\ref{thm:aucl}, which states that
\[
\AUC_L=\mathbb{P}(L_1>U_0)=\mathbb{P}(I_1>I_0).
\]
\paragraph{Proof.}
Let \(G_1(l)=\mathbb{P}(L_1\le l)\) and \(H_0(u)=\mathbb{P}(U_0\le u)\). By definition,
\[
\mathrm{TPR}_L(t)=\mathbb{P}(L_1>t)=1-G_1(t),
\qquad
\mathrm{FPR}_U(t)=\mathbb{P}(U_0>t)=1-H_0(t).
\]
Therefore,
\[
\begin{aligned}
\AUC_L
&=\int \mathrm{TPR}_L(t)\,d[\mathrm{FPR}_U(t)]\\
&=\int \bigl(1-G_1(t)\bigr)\,d\bigl(1-H_0(t)\bigr)\\
&=\mathbb{E}_{U_0}\{1-G_1(U_0)\}\\
&=\mathbb{P}(L_1>U_0).
\end{aligned}
\]
The final equality follows because \(1-G_1(U_0)=\mathbb{P}(L_1>U_0\mid U_0)\), and \(L_1\) and \(U_0\) are drawn independently from the positive and negative classes. Since \(I_1>I_0\) is defined as \(L_1>U_0\), we obtain
\[
\AUC_L=\mathbb{P}(L_1>U_0)=\mathbb{P}(I_1>I_0).
\]

We prove Theorem~\ref{thm:aucu}, which states that
\[
\AUC_U=1-\mathbb{P}(U_1<L_0)=1-\mathbb{P}(I_1<I_0).
\]

\paragraph{Proof.}
Let \(H_1(u)=\mathbb{P}(U_1\le u)\) and \(G_0(l)=\mathbb{P}(L_0\le l)\). By definition,
\[
\mathrm{TPR}_U(t)=\mathbb{P}(U_1>t)=1-H_1(t),
\qquad
\mathrm{FPR}_L(t)=\mathbb{P}(L_0>t)=1-G_0(t).
\]
Therefore, using the same Stieltjes-integral argument as in the proof of Theorem~\ref{thm:aucl},
\[
\AUC_U=
\int \bigl(1-H_1(t)\bigr)\,d\bigl(1-G_0(t)\bigr)
=\mathbb{E}_{L_0}\{1-H_1(L_0)\}
=\mathbb{P}(U_1>L_0).
\]
Under \(\mathbb{P}(U_1=L_0)=0\), we have
\[
\mathbb{P}(U_1>L_0)=1-\mathbb{P}(U_1<L_0).
\]

Since \(I_1<I_0\) is defined as \(U_1<L_0\), it follows that
\[
\AUC_U=1-\mathbb{P}(U_1<L_0)=1-\mathbb{P}(I_1<I_0).
\]

\subsection{Proof of Theorem~\ref{thm:main_bound}}
\label{app:proof_main_bound}

We prove Theorem~\ref{thm:main_bound}. Recall that \(I(x)=[L(x),U(x)]\) is an interval-valued prediction for the posterior risk \(\eta(x)=\mathbb{P}(Y=1\mid X=x)\). Let
\[
E_{\mathrm{pair}}=\{\eta(X_1)\in I(X_1)\}\cap\{\eta(X_0)\in I(X_0)\}
\]
be the event that both intervals cover their corresponding posterior risks. Under the class-conditional miscoverage rates \(\alpha_1\) and \(\alpha_0\), the pairwise miscoverage probability is
\[
p_{\mathrm{pair}}=\mathbb{P}(E_{\mathrm{pair}}^c)=1-(1-\alpha_1)(1-\alpha_0)=\alpha_1+\alpha_0-\alpha_1\alpha_0.
\]

Let
\[
B=\{\eta(X_1)>\eta(X_0)\}
\]
denote the Bayes-correct ordering event, so that \(\AUC^*=\mathbb{P}(B)\). Let
\[
A=\{I_1>I_0\}=\{L_1>U_0\},\qquad C=\{I_1<I_0\}=\{U_1<L_0\}.
\]
On the pairwise coverage event \(E_{\mathrm{pair}}\), if \(A\) occurs then
\[
\eta(X_1)\ge L_1>U_0\ge \eta(X_0),
\]
so \(A\cap E_{\mathrm{pair}}\subseteq B\). Therefore,
\[
\mathbb{P}(B)\ge \mathbb{P}(A\cap E_{\mathrm{pair}})\ge \mathbb{P}(A)-\mathbb{P}(E_{\mathrm{pair}}^c)=\AUC_L-p_{\mathrm{pair}}.
\]
This proves the lower bound.

For the upper bound, observe that on \(E_{\mathrm{pair}}\), the strict reverse interval ordering \(C\) implies
\[
\eta(X_1)\le U_1<L_0\le \eta(X_0),
\]
so \(C\cap E_{\mathrm{pair}}\subseteq B^c\). Hence \(B\cap E_{\mathrm{pair}}\subseteq C^c\), and
\[
\mathbb{P}(B)\le \mathbb{P}(C^c)+\mathbb{P}(E_{\mathrm{pair}}^c)=1-\mathbb{P}(C)+p_{\mathrm{pair}}.
\]
By Theorem~\ref{thm:aucu}, \(1-\mathbb{P}(C)=\AUC_U\), so
\[
\AUC^*=\mathbb{P}(B)\le \AUC_U+p_{\mathrm{pair}}.
\]
Combining the two inequalities gives
\[
\AUC_L-p_{\mathrm{pair}}\le \AUC^*\le \AUC_U+p_{\mathrm{pair}}.
\]

\section{Additional experimental details and results}
\label{app:experiments}

This appendix provides additional details for the experiments in
Section~\ref{sec:experiments}. We include numerical validation of the
AUC probability identities, full Pima model-generator results, additional
MIMIC-IV cohort details, and the synthetic validation of the optimal-AUC bound.

\subsection{Compute resources}
\label{app:compute_resources}

All experiments were run locally on a 2020 MacBook Pro with an Apple M1 chip and 8GB of memory. No GPU acceleration was used. The benchmark experiments and synthetic validation can be run on a standard laptop. The MIMIC-IV bootstrap experiment is the most computationally intensive component because each model is refit across bootstrap resamples. We used fixed random seeds for all data splits, bootstrap resampling, and model initialization when applicable.

\subsection{Numerical validation and details for Figure~\ref{fig:roc_geometry_decomposition}}
\label{app:empirical_validation}

We use one illustrative bootstrap interval example to generate Figure~\ref{fig:roc_geometry_decomposition} and to numerically validate the AUC--probability identities in Theorems~\ref{thm:aucl} and~\ref{thm:aucu}. The example uses the Pima Indians Diabetes dataset from OpenML, with a stratified 30\% training and 70\% test split using random seed 42. Zero entries in the physiologically implausible columns \texttt{plas}, \texttt{pres}, \texttt{skin}, \texttt{insu}, and \texttt{mass} are treated as missing and median-imputed. The test set contains 538 subjects, with a positive-class fraction of 0.349.

Interval-valued predictions are obtained from 100 stratified bootstrap refits of an \(L_2\)-regularized logistic regression model with standard feature scaling. For each test subject, the 90\% interval is defined by the empirical 5th and 95th percentiles of the bootstrap predicted risks. We compute \(\AUC_L\) by trapezoidal integration of the \(\mathrm{TPR}_L\) versus \(\mathrm{FPR}_U\) curve and compare it with the empirical pairwise estimate \(P(I_1>I_0)\). Similarly, we compute \(\AUC_U\) from the upper interval ROC curve and compare it with \(1-P(I_1<I_0)\).

\begin{table}[h]
\centering
\caption{Numerical validation of AUC--probability equivalence at the 90\% nominal interval level.}
\label{tab:validation_90ci_app}
\begin{tabular}{lrr}
\toprule
Quantity & Value & Difference \\
\midrule
$\AUC_L$ vs. $P(I_1>I_0)$ & 0.6074 / 0.6072 & 0.0002 \\
$\AUC_U$ vs. $1-P(I_1<I_0)$ & 0.9427 / 0.9439 & 0.0013 \\
\midrule
$P(I_1>I_0)$ & 0.6072 & --- \\
$P(\mathrm{overlap})$ & 0.3367 & --- \\
$P(I_1<I_0)$ & 0.0561 & --- \\
Sum of probabilities & 1.0000 & --- \\
\bottomrule
\end{tabular}
\end{table}

The numerical discrepancies are small and arise from numerical integration of the empirical ROC curves. The three regions also sum to one, confirming the empirical decomposition \(P(I_1>I_0)+P(\mathrm{overlap})+P(I_1<I_0)=1\). This is the same interval-valued prediction example used to draw the geometric regions in Figure~\ref{fig:roc_geometry_decomposition}(a).

For the threshold-ambiguity panel in Figure~\ref{fig:roc_geometry_decomposition}(b), the displayed threshold is \(t=0.3886\), chosen to give a clear triangle in the interior of the ROC plane. The corresponding points are
\[
A=(\mathrm{FPR}_U,\mathrm{TPR}_L)=(0.4286,0.5851),
\]
\[
B=(\mathrm{FPR}_L,\mathrm{TPR}_U)=(0.1314,0.8723),\qquad
C=(\mathrm{FPR}_U,\mathrm{TPR}_U)=(0.4286,0.8723).
\]
Thus the vertical leg has length
\[
\mathrm{TPR}_U(t)-\mathrm{TPR}_L(t)=0.2872=P(L_1\le t\le U_1),
\]
and the horizontal leg has length
\[
\mathrm{FPR}_U(t)-\mathrm{FPR}_L(t)=0.2971=P(L_0\le t\le U_0).
\]
The right angle is at \(C\), so the triangle visualizes the positive- and negative-class threshold ambiguity at the selected threshold.

\subsection{Benchmark experiment details}
\label{app:benchmark_details}

The benchmark experiment in Section~\ref{subsec:benchmark} uses the Pima Indians Diabetes dataset. The dataset contains 768 female patients aged 21 and above, with eight clinical features, and the binary outcome indicates diabetes onset within five years. We use a stratified 30/70 train-test split with random seed 42, giving 230 training samples and 538 test samples. Zero entries in the standard physiologically implausible columns are treated as missing and median-imputed. All reported model fits, bootstrap samples, ensemble refits, and posterior draws use fixed random seeds.

We compare five base models: logistic regression (LR), random forest (RF), histogram gradient boosting (HistGBM), multilayer perceptron (MLP), and explainable boosting machine (EBM). LR and MLP are trained in a \texttt{StandardScaler} pipeline. The main hyperparameters are: LR uses \(L_2\) regularization with \texttt{max\_iter=500}; RF uses 200 trees with maximum depth 8; HistGBM uses 200 boosting iterations, maximum depth 3, learning rate 0.1, and early stopping; MLP uses hidden layers of sizes 32 and 16 with ReLU activations; EBM uses \texttt{outer\_bags=2} and no interaction terms. Models are implemented in \texttt{scikit-learn} except for EBM, which uses the \texttt{interpret} package \citep{scikit-learn,nori2019interpretml}.

We compare four interval-generation procedures. Bootstrap intervals are formed from 200 stratified bootstrap refits of the base model. Deep ensemble intervals are formed from 50 refits with different random seeds, and are omitted for deterministic estimators such as LR. Bayesian bootstrap intervals use 200 Dirichlet-weighted refits following Rubin's Bayesian bootstrap \citep{rubin1981bayesian}; for estimators without sample-weight support, weighted resampling is used. Finally, a simple model-intrinsic baseline is included when available: a Laplace approximation for logistic regression and Monte Carlo dropout for the MLP \citep{gal2016dropout}. All interval endpoints are empirical percentiles of the resulting predicted risk distributions at the specified nominal interval level.

For each test instance \(x_i\), each method produces an interval-valued risk prediction \(I(x_i)=[L_i,U_i]\). For each positive-negative test pair, we classify the pair as \(I_1>I_0\) if \(L_1>U_0\), as \(I_1<I_0\) if \(U_1<L_0\), and as overlapping otherwise. The reported quantities are empirical fractions over all positive-negative test pairs.

\subsection{Full Pima model-generator results}
\label{app:pima_full_results}

Figure~\ref{fig:pima_full_sweep} shows the full model-generator grid across nominal interval levels, extending the selected panels in Figure~\ref{fig:pima_sweep}. Table~\ref{tab:pima_full_fixed} reports the corresponding fixed-level decomposition at the 90\% nominal interval level.

\begin{figure}[tb]
    \centering
    \includegraphics[width=\linewidth]{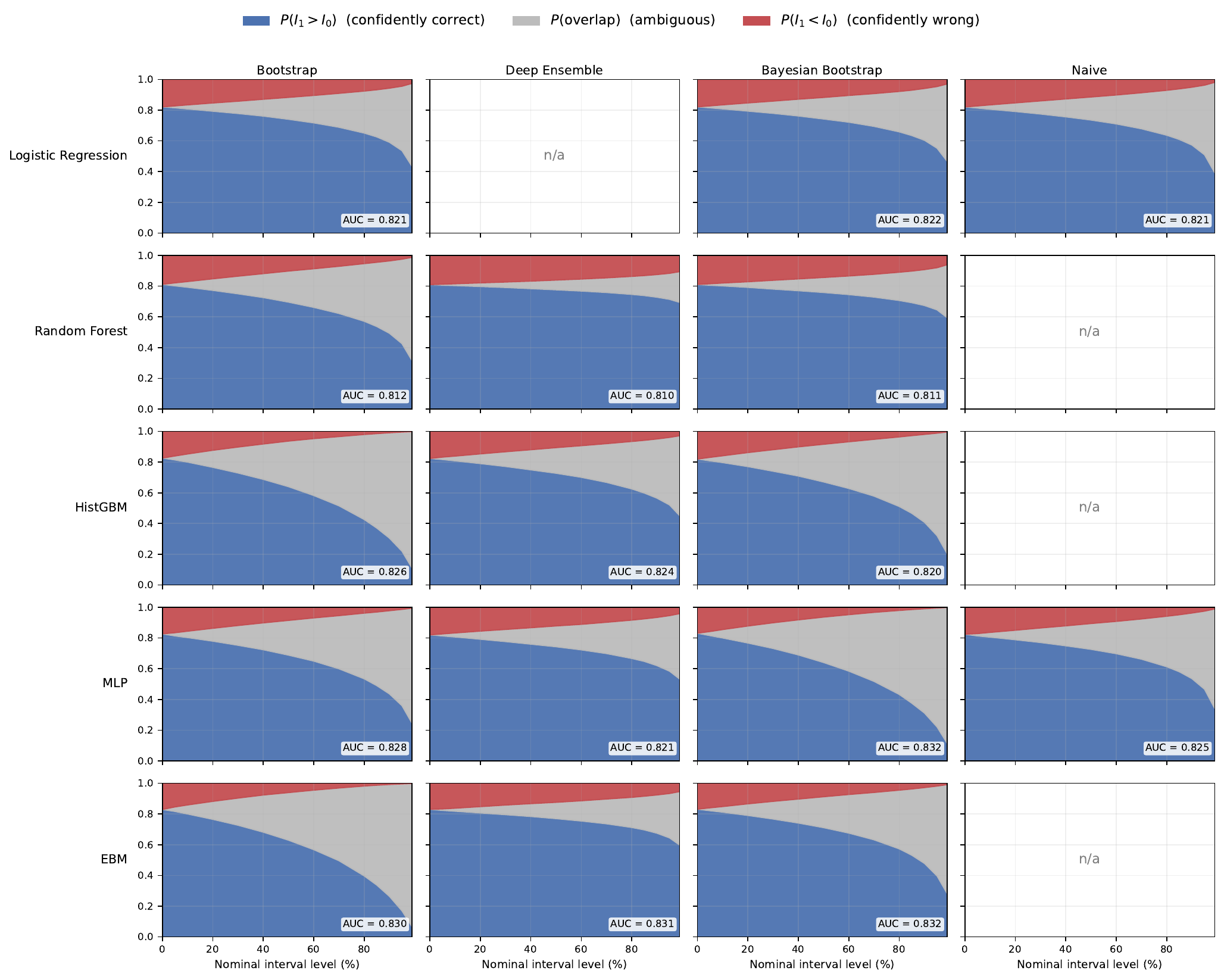}
    \caption{Full three-region decomposition on the Pima Indians Diabetes benchmark across nominal interval levels. Rows correspond to base models and columns correspond to interval-generators.}
    \label{fig:pima_full_sweep}
\end{figure}

\begin{table}[tb]
    \centering
    \caption{Full Pima Indians Diabetes decomposition at the 90\% nominal interval level.}
    \label{tab:pima_full_fixed}
    \resizebox{\linewidth}{!}{\begin{tabular}{llrrrrrrr}
\toprule
Model & Generator & Point AUC & $P(I_1>I_0)$ & $P(\mathrm{overlap})$ & $P(I_1<I_0)$ & $\mathrm{AUC}_L$ & $\mathrm{AUC}_U$ & Sum \\
\midrule
LR & Bootstrap & 0.821 & 0.591 & 0.352 & 0.057 & 0.591 & 0.943 & 1.000 \\
LR & Bayesian Bootstrap & 0.822 & 0.604 & 0.337 & 0.059 & 0.604 & 0.941 & 1.000 \\
LR & Naive & 0.821 & 0.573 & 0.377 & 0.051 & 0.573 & 0.950 & 1.000 \\
RF & Bootstrap & 0.812 & 0.493 & 0.471 & 0.036 & 0.493 & 0.964 & 1.000 \\
RF & Deep Ensemble & 0.810 & 0.728 & 0.147 & 0.125 & 0.728 & 0.875 & 1.000 \\
RF & Bayesian Bootstrap & 0.811 & 0.674 & 0.234 & 0.092 & 0.674 & 0.907 & 1.000 \\
HistGBM & Bootstrap & 0.826 & 0.305 & 0.685 & 0.009 & 0.305 & 0.991 & 1.000 \\
HistGBM & Deep Ensemble & 0.824 & 0.566 & 0.385 & 0.049 & 0.566 & 0.951 & 1.000 \\
HistGBM & Bayesian Bootstrap & 0.820 & 0.409 & 0.572 & 0.020 & 0.409 & 0.980 & 1.000 \\
MLP & Bootstrap & 0.828 & 0.437 & 0.541 & 0.022 & 0.437 & 0.978 & 1.000 \\
MLP & Deep Ensemble & 0.821 & 0.621 & 0.314 & 0.066 & 0.621 & 0.934 & 1.000 \\
MLP & Bayesian Bootstrap & 0.832 & 0.312 & 0.679 & 0.009 & 0.312 & 0.991 & 1.000 \\
MLP & Naive & 0.825 & 0.536 & 0.427 & 0.037 & 0.536 & 0.963 & 1.000 \\
EBM & Bootstrap & 0.830 & 0.265 & 0.728 & 0.007 & 0.265 & 0.993 & 1.000 \\
EBM & Deep Ensemble & 0.831 & 0.675 & 0.249 & 0.076 & 0.675 & 0.924 & 1.000 \\
EBM & Bayesian Bootstrap & 0.832 & 0.479 & 0.493 & 0.028 & 0.479 & 0.972 & 1.000 \\
\bottomrule
\end{tabular}
}
\end{table}

\subsection{MIMIC-IV experiment details}
\label{app:mimic_details}

We use MIMIC-IV for a real-world in-hospital mortality prediction experiment. The cohort contains 54,473 adult ICU stays, restricted to the first ICU stay per hospital admission, age at least 18 years, and ICU length of stay at least 24 hours. The outcome is in-hospital mortality, defined by \texttt{admissions.hospital\_expire\_flag}. The cohort includes 5,766 deaths. We use a stratified 70/30 train-test split at the stay level with random seed 42, yielding 38,131 training stays and 16,342 test stays. The mortality rate is 0.1058 in the training set and 0.1059 in the test set.

Predictors are extracted from the first 24 hours of each ICU stay. We use 47 features: age, gender, and 15 laboratory variables summarized by mean, minimum, and maximum. The laboratory variables are glucose, creatinine, blood urea nitrogen, lactate, sodium, potassium, chloride, bicarbonate, anion gap, white blood cells, hemoglobin, hematocrit, platelets, arterial pH, and arterial oxygen pressure. Values outside prespecified physiologic ranges are clipped before aggregation. Missing values are imputed using training-set medians, and continuous features are standardized using a \texttt{StandardScaler} fitted on the training set. The same preprocessing pipeline and train-test split are used for all models.

We compare four predictive models: \(L_2\)-regularized logistic regression (LR), random forest (RF), histogram gradient boosting (HistGB), and multilayer perceptron (MLP). For each model, interval-valued predictions are constructed using percentile bootstrap resampling. We draw \(B=300\) bootstrap samples from the training set, refit the model on each resampled dataset, and predict risks on the fixed held-out test set. The point risk score is the mean bootstrap predicted probability. For the main 90\% interval results, the interval endpoints are the empirical 5th and 95th percentiles of the bootstrap predicted risks.

For each model, all positive-negative test pairs are enumerated exactly. At the 90\% interval level, this corresponds to approximately 25 million positive-negative comparisons per model. Each pair is classified as confidently correct if \(I_1>I_0\), confidently incorrect if \(I_1<I_0\), and ambiguous otherwise. Table~\ref{tab:mimic_model_comparison} reports point AUC, \(\AUC_L=P(I_1>I_0)\), \(\AUC_U=1-P(I_1<I_0)\), \(P(\mathrm{overlap})\).

\begin{table}[t]
\centering
\caption{MIMIC-IV model comparison at the 90\% nominal interval level.}
\label{tab:mimic_model_comparison}
\begin{tabular}{lrrrrr}
\toprule
Model & AUC & $\AUC_L$ & $\AUC_U$ & $P(\mathrm{overlap})$ & $P(I_1<I_0)$ \\
\midrule
LR & 0.772 & 0.678 & 0.848 & 0.170 & 0.152 \\
RF & 0.814 & 0.637 & 0.925 & 0.288 & 0.075 \\
HistGB & 0.835 & 0.567 & 0.962 & 0.395 & 0.038 \\
MLP & 0.827 & 0.405 & 0.985 & 0.581 & 0.015 \\
\bottomrule
\end{tabular}
\end{table}

The point-AUC ranking differs from the interval-aware decomposition. HistGB has the largest point AUC, whereas LR has the largest \(\AUC_L\), and MLP has the smallest confidently incorrect fraction but the largest overlap fraction.

\begin{table}[t]
\centering
\caption{Risk-aware model selection for MIMIC-IV using \(\max_m P_m(I_1>I_0)\) subject to \(P_m(I_1<I_0)\le \epsilon\).}
\label{tab:mimic_model_selection}
\begin{tabular}{lll r}
\toprule
$\epsilon$ & Eligible models & Selected model & Selected $\AUC_L$ \\
\midrule
0.05 & HistGB, MLP & HistGB & 0.567 \\
0.10 & RF, HistGB, MLP & RF & 0.637 \\
0.15 & RF, HistGB, MLP & RF & 0.637 \\
\bottomrule
\end{tabular}
\end{table}

For instance-level selective prediction, we evaluate the threshold-abstention rule at two thresholds: the training-set mortality prevalence \(t=0.1058\), and an illustrative fixed threshold \(t=0.20\). A test instance is classified positive if \(L(x)>t\), classified negative if \(U(x)<t\), and the model abstains if \(L(x)\le t\le U(x)\). Table~\ref{tab:mimic_threshold_abstention} reports standard point-score TPR/FPR, abstention rate, retained-subject fraction, class-specific abstention rates, and TPR/FPR among retained subjects.

\begin{table}[t]
\centering
\small
\caption{MIMIC-IV threshold-abstention results at the 90\% nominal interval level.}
\label{tab:mimic_threshold_abstention}
\resizebox{\linewidth}{!}{%
\begin{tabular}{llrrrrrrr}
\toprule
Model & Threshold & Point TPR & Point FPR & Abstain & Retained & Pos. abstain & Neg. abstain & Retained TPR/FPR \\
\midrule
LR & Train prev. & 0.664 & 0.263 & 0.137 & 0.863 & 0.155 & 0.135 & 0.683 / 0.228 \\
RF & Train prev. & 0.722 & 0.261 & 0.248 & 0.752 & 0.265 & 0.246 & 0.777 / 0.195 \\
HistGB & Train prev. & 0.741 & 0.241 & 0.317 & 0.683 & 0.405 & 0.306 & 0.801 / 0.141 \\
MLP & Train prev. & 0.757 & 0.274 & 0.476 & 0.524 & 0.606 & 0.461 & 0.824 / 0.093 \\
LR & \(t=0.20\) & 0.343 & 0.078 & 0.057 & 0.943 & 0.120 & 0.049 & 0.333 / 0.058 \\
RF & \(t=0.20\) & 0.437 & 0.082 & 0.134 & 0.866 & 0.291 & 0.115 & 0.410 / 0.044 \\
HistGB & \(t=0.20\) & 0.499 & 0.093 & 0.214 & 0.786 & 0.453 & 0.186 & 0.463 / 0.032 \\
MLP & \(t=0.20\) & 0.485 & 0.102 & 0.319 & 0.681 & 0.621 & 0.283 & 0.438 / 0.016 \\
\bottomrule
\end{tabular}}
\end{table}

Across both thresholds, models with lower FPR among retained subjects generally abstain on a larger fraction of subjects. For example, at the training-prevalence threshold, MLP has the lowest retained-subject FPR, but abstains on 47.6\% of test subjects. This reflects the expected trade-off between selective operating performance and retained-subject fraction.

All quality checks passed. Predicted risks and interval endpoints are in \([0,1]\), all intervals satisfy \(L\le U\), and \(P(I_1>I_0)+P(\mathrm{overlap})+P(I_1<I_0)=1\) up to numerical precision for every model and interval level. The reported \(\AUC_L\) and \(\AUC_U\) also match \(P(I_1>I_0)\) and \(1-P(I_1<I_0)\) up to numerical precision.

\subsection{Synthetic validation of optimal AUC bounds}
\label{app:bound_validation}

We empirically validate Theorem~\ref{thm:main_bound} using a synthetic
experiment in which the true posterior \(\eta(x)\) is known, allowing direct
calculation of the Bayes-optimal AUC.

\paragraph{Setup.}
We consider a binary classification task with a one-dimensional feature \(X\)
following
\[
X\mid Y=i \sim \mathcal{N}(\mu_i,1).
\]
To study a high-ambiguity regime where uncertainty-aware evaluation is most
relevant, we set \(\mu_0=0\) and \(\mu_1=1\), yielding a ground-truth
\(\AUC^*=0.764\). For a given miscoverage rate \(\alpha\), we
construct prediction intervals \(I(x)=[L(x),U(x)]\) centered at \(\eta(x)\) for
a \(1-\alpha\) fraction of samples. Miscoverage is induced for the remaining
\(\alpha\) fraction by shifting the interval centers. All interval boundaries
are clipped to \([0,1]\).

\begin{figure}[h]
\centering
\includegraphics[width=0.7\linewidth]{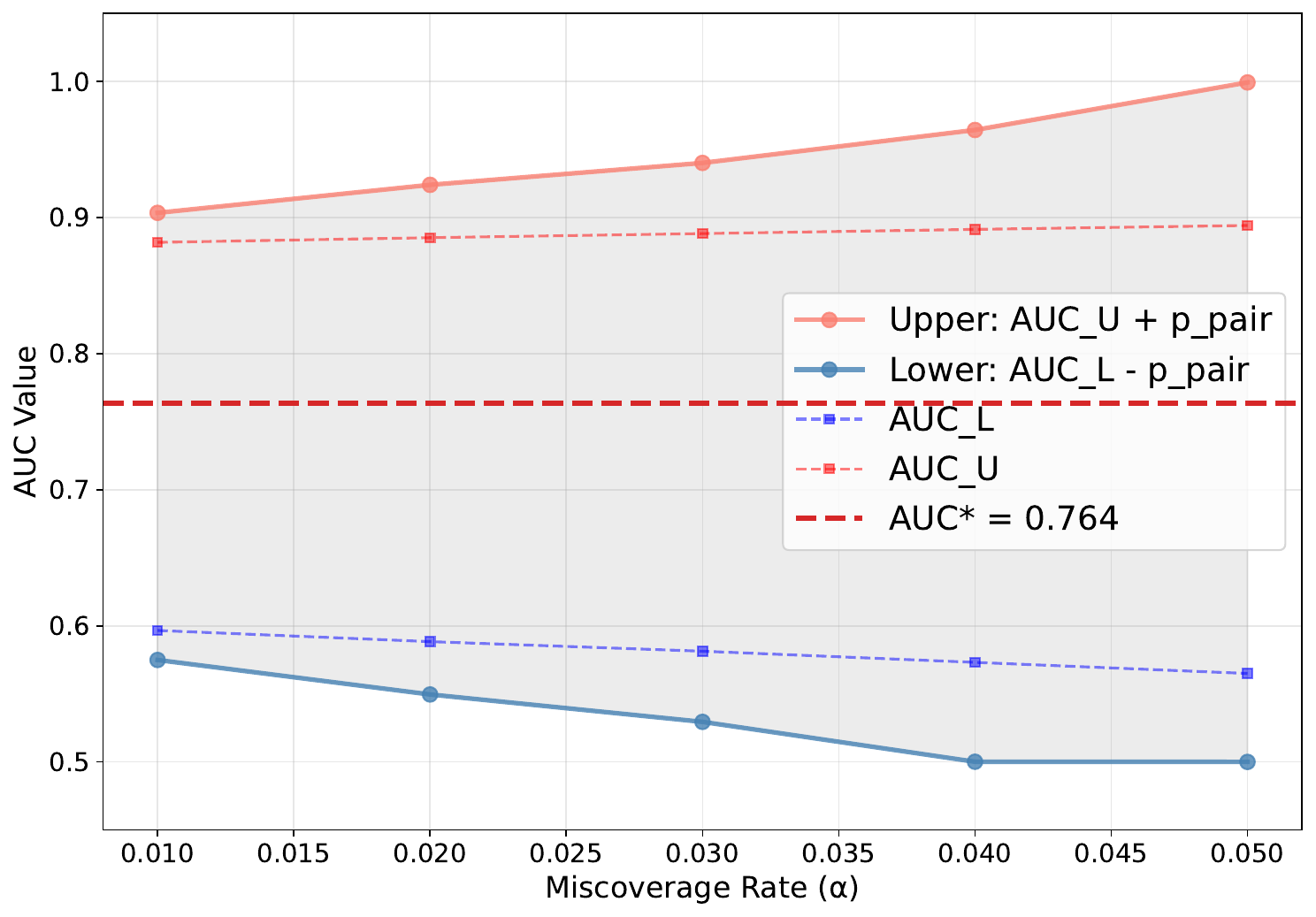}
\caption{Synthetic validation of theoretical bounds on the optimal AUC
\(\AUC^*\). The shaded region denotes the theoretical range
\([\AUC_L-p_{\mathrm{pair}},\AUC_U+p_{\mathrm{pair}}]\), which contains the
Bayes-optimal AUC across evaluated miscoverage rates.}
\label{fig:bound_app}
\end{figure}

\paragraph{Results.}
As shown in Figure~\ref{fig:bound_app}, the optimal \(\AUC^*\) is contained
within the theoretical bounds across all evaluated miscoverage rates. As
\(\alpha\) increases, the bounds widen, reflecting the explicit dependence on
the pairwise miscoverage probability \(p_{\mathrm{pair}}\). This widening is
expected in high-overlap settings, where a large fraction of pairwise
comparisons cannot be resolved without stronger assumptions.

The width of these bounds reflects the conservatism of distribution-free
guarantees. Rather than providing a tight estimate of achievable performance,
the bounds characterize the range of Bayes-optimal discrimination compatible
with the observed interval uncertainty. In this regime, the gap between
\(\AUC_L\) and \(\AUC_U\) serves as a quantitative indicator of ranking
ambiguity.

\end{document}